\documentclass[10pt,journal,compsoc]{IEEEtran}
\ifCLASSOPTIONcompsoc
  \usepackage[nocompress]{cite}
\else
  \usepackage{cite}
\fi
\ifCLASSINFOpdf
\else
\fi
\usepackage[utf8]{inputenc} 
\usepackage[T1]{fontenc}    
\usepackage{hyperref}       
\usepackage{url}            
\usepackage{booktabs}       
\usepackage{amsfonts}       
\usepackage{nicefrac}       
\usepackage{microtype}      
\usepackage{amsthm}
\usepackage{graphicx}
\usepackage{amsmath,amssymb,amsfonts}
\usepackage{algorithm,algorithmic}
\usepackage{subfigure}
\usepackage{comment}
\usepackage{color}
\usepackage{soul}
\usepackage[dvipsnames]{xcolor}

\def\ab{{\mathbf a}}
\def\db{{\mathbf d}}
\def\xb{{\mathbf x}}
\def\cb{{\mathbf c}}
\def\zb{{\mathbf z}}

\def\Ab{{\mathcal A}}
\def\Ib{{\mathbf I}}
\def\Pb{{\mathcal P}}
\def\Mb{{\mathcal M}}
\def\Fb{{\mathbf F}}
\def\Db{{\mathbf D}}
\def\Qb{{\mathbf Q}}
\def\Wb{{\mathbf W}}

\def\Zb{{\mathbf Z}}
\def\Ub{{\mathbf U}}
\newcommand{\argmind}[2]{\ensuremath{\underset{\substack{{#1}}}%
{\mathrm{argmin}}\;\;#2 }}
\newcommand{\minimize}[2]{\ensuremath{\underset{\substack{{#1}}}%
{\text{\rm minimize}}\;\;#2}}

\newtheorem{theorem}{Theorem}[section]

\begin{document}
%
\title{Deep Transform and Metric Learning Network: Wedding Deep Dictionary Learning and Neural Network}
%
%
%
%

\author{Wen Tang,~\IEEEmembership{Member,~IEEE,}
        Emilie Chouzenoux$^\dagger$,~\IEEEmembership{{Senior Member,~IEEE,}}
        Jean-Christophe Pesquet$^\dagger$,~\IEEEmembership{{Fellow,~IEEE}}
        ~and Hamid Krim,~\IEEEmembership{{Fellow, ~IEEE}}
\IEEEcompsocitemizethanks{\IEEEcompsocthanksitem W. Tang and H. Krim are with the Department
of Electrical and Computer Engineering, North Carolina State University, Raleigh, NC, 27606.\protect\\
E-mail: \{wtang6,~ahk\}@ncsu.edu
\IEEEcompsocthanksitem {$\dagger$ E. Chouzenoux and J. Pesquet are with Universit\'e Paris-Saclay, CentraleSup\'elec, Center for Visual Computing, Inria, 91190 Gif sur Yvette, France.}\protect\\
E-mail: \{emilie.chouzenoux,~jean-christophe.pesquet\}@centralesupelec.fr
}
}

\IEEEtitleabstractindextext{%
\begin{abstract}
On account of its many successes in inference tasks and imaging applications, Dictionary Learning (DL) and its related sparse optimization problems have garnered a lot of research interest.
While most solutions have focused on single layer dictionaries, the improved recently proposed Deep DL (DDL) methods have also fallen short on a number of issues.
We propose herein, a novel DDL approach where each DL layer 
can be formulated as a combination of one linear layer and a Recurrent Neural Network (RNN).
The RNN 
is shown to flexibly account for the layer-associated and learned metric.
Our proposed work unveils new insights into Neural Networks and DDL and provides a new, efficient and competitive approach to jointly learn a deep transform and a metric for inference applications.
Extensive experiments {on image classification problems} are carried out to demonstrate that the proposed method can not only outperform existing DDL but also state-of-the-art generic CNNs and also achieve better robustness against adversarial perturbations.
\end{abstract}

}

\maketitle

\IEEEdisplaynontitleabstractindextext

%
\IEEEpeerreviewmaketitle

\IEEEraisesectionheading{\section{Introduction}\label{sec:introduction}}

%
%
%
%
\IEEEPARstart{D}{ictionary} Learning/Sparse Coding has demonstrated its high potential in exploring the semantic information embedded in high dimensional noisy data. It has been successfully applied for solving different inference tasks, such as image denoising \cite{denoising}, image restoration \cite{imagerestoration}, image super-resolution \cite{superresolution,skau2016pansharpening}, audio processing \cite{audioprocessing} and image classification \cite{imageclassification}.

While Synthesis Dictionary Learning (SDL) has been greatly investigated and widely used, the Analysis Dictionary Learning (ADL)/Transform Learning, as a dual problem, has been getting greater attention for its robustness property among others \cite{analysisksvd,xiao16,tang2016analysis}.
DL based methods 
have primarily focused on learning one-layer dictionary and its associated sparse representation. Other variations on the classification theme have also been appearing with a goal of addressing some recognized limitations, such as task-driven dictionary learning \cite{mairal2009supervised}, first introduced to jointly learn the dictionary, its sparse representation, and its classification objective. In \cite{ksvd}, a label consistent term is additionally considered. Class-specific dictionary learning has been recently shown to improve the discrimination in \cite{ramirez2010classification,FDDL,wang2013max} at the expense of a higher complexity. On the ADL side, more and more efficient classifiers \cite{dadl,cadl,sksvdadl,tang2018structured,tang2019convolution} have resulted from numerous research efforts, and 
have yielded to an outperformance of SDL in both training and testing phases~\cite{tang2019analysis}.


DL methods with their associated sparse representation, present significant computational challenges addressed by different techniques, including K-SVD \cite{ksvd,analysisksvd}, SNS-ADL \cite{xiao16} and Fast Iterative Shrinkage-thresholding Algorithm (FISTA) \cite{fista}. Meant to provide a practically faster solution, the alternating minimization of FISTA still exhibited limitations and a relatively high computational cost.



To address these computational and scaling difficulties, differentiable programming solutions have also been developed, to take advantage of the efficiency of neural networks. LISTA \cite{LISTA} was first proposed to unfold iterative hard-thresholding into an RNN format, thus speeding up SDL. Unlike conventional solutions for solving optimization problems, LISTA uses the forward and backward passes to simultaneously update the sparse representation and dictionary in an efficient manner. In the same spirit, sparse LSTM (SLSTM) \cite{zhou2018sc2net} adapts LISTA to a Long Short Term Memory structure to automatically learn the dimension of the sparse representation.

Although the aforementioned differentiable programming methods are efficient at solving a single-layer DL problem, the latter formulation still does not yield the best performance in image classification tasks. 
With the fast development of deep learning, Deep Dictionary Learning (DDL) methods \cite{tariyal2016deep,shahin2017image} have thus come into play. In \cite{huang0418_DDSR}, a deep model for ADL followed by a SDL is developed for image super-resolution. Also, \cite{mahdizadehaghdam2019deep} deeply stacks SDLs to classify images by achieving promising and robust results. Unsupervised DDL approaches have also been proposed, with promising results \cite{Maggu2018,Gupta}.

However, to the best of our knowledge, 
no DDL model which can provide both a fast and reliable solution has been proposed.
\begin{figure*}[!htb]
    \centering
    \includegraphics[width=0.85\textwidth]{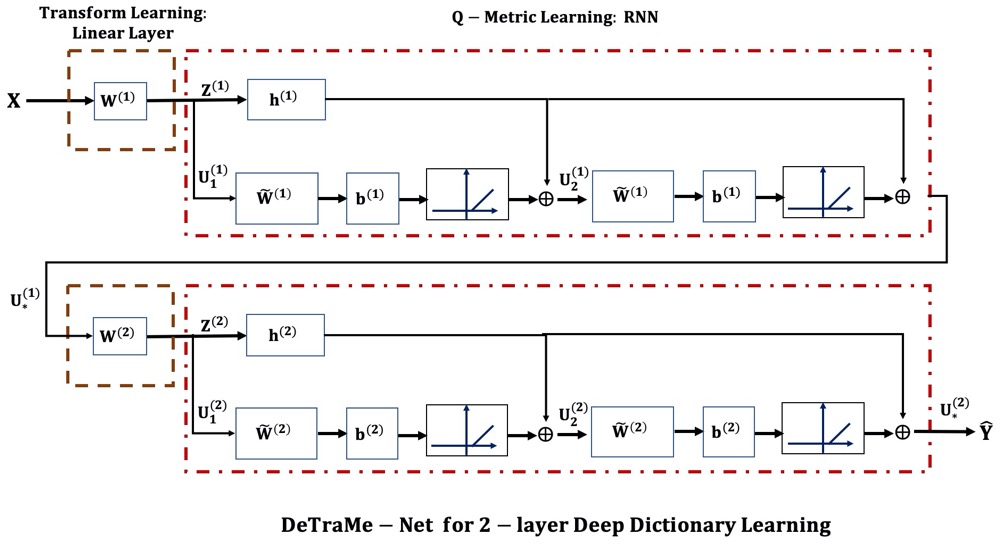}
    \caption{2-layer DeTraMe-Net model. Each layer solves a DL problem, which is transformed into the combination of Transforming Learning (\textit{i.e.}, linear layer in brown dashed lines) and Q-Metric Learning (\textit{i.e.}, RNN in red dashed dot lines). A truncated 2-iterations RNN is unfolded. Sparsity is imposed by shifted-ReLU functions. In the forward pass, we first use a linear layer to learn the new representation $\mathbf{Z^{(1)}}$ for input data $\mathbf{X}$. The RNN is then used to iteratively learn the optimal sparse representation $\mathbf{U^{(1)}_*}$. For the second layer, the sparse representation $\mathbf{U^{(1)}_*}$ is used as input to learn the second layer sparse representation $\mathbf{U^{(2)}_*}$. Finally, a cross-entropy loss based on $\mathbf{U^{(2)}_*}$ and ground truth $\mathbf{Y}$ is used. The parameters $\mathbf{W}^{(i)}$, $\widetilde{\mathbf{W}}^{(i)}$, $\mathbf{h}^{(i)}$ and $\mathbf{b}^{(i)}$, $i=1,2$, in the linear layer and RNN parts are learned by back-propagation.}
    \label{fig:framework}
\end{figure*}
The proposed work herein, aims at ensuring the discriminative ability of single-layer DL while 
providing the efficiency of end-to-end models. To this end, we propose a novel differentiable programming method 
to jointly learn a deep metric together with an associated transform. Cascading these canonical structures will exploit and strengthen the structure learning capacity of a deep network, yielding what we refer to a Deep Transform and Metric Learning Network (DeTraMe-Net).
This newly proposed approach not only increases the discrimination capabilities of DL, but also affords a flexibility of constructing different DDL or Deep Neural Network (DNN) architectures. 
As will be later shown, this approach also resolves usually arising initialization and gradient propagation issues in DDL.

As shown in Figure \ref{fig:framework}, in each layer of DeTraMe-Net, the DL problem is decomposed as a transform learning
one, \textit{i.e.} a linear layer part cascaded with a nonlinear component using a learned metric. The latter, referred to as  Q-Metric Learning, 
is realized by an RNN.
One of the contributions of our work is to show
how DDL can theoretically be reformulated as such a combination of linear layers and RNNs. 
Decoupling the metric and the dual frame operator (pseudo-inverse of dictionary) into two independent variables is also shown to introduce additional flexibility, and to improve the power of DL.  
On the practical side, and to achieve a faster and simpler implementation, we impose a block-diagonal structure for Q-Metric Learning leading to parallel processing of independent channels. Moreover, a convolutional operator is also introduced to decrease the number of parameters,
thus leading to a Convolutional-RNN. Additionally, the Q-Metric Learning part may be viewed as a non-separable activation function that can be flexibly included into any architecture.
As a result, different new DeTraMe networks may be obtained by integrating Q-Metric Learning into various CNN architectures such as Plain CNN \cite{ALLCNN} and ResNets \cite{resnet}.
The resulting DeTraMe-Nets-based architectures are demonstrated to be more discriminative than generic CNN models.

Although the authors of \cite{wang2015deep} and \cite{liu2018dictionary} also used a CNN followed by an RNN for respectively solving super-resolution and sense recognition tasks, they directly used LISTA in their model. In turn, our method
actually solves the same problem as LISTA. In addition, in
\cite{wang2015deep} and \cite{liu2018dictionary},
a sparse representation was jointly learned, while a more discriminative DDL approach is achieved in our work.
We also formally derive the linear and RNN-based layer structure from DDL, thus providing a theoretical justification and a rationale to such approaches. This may also open an avenue to new and more creative and performing alternatives. 
We recently discovered that independently, a L1 norm transformation was used in conjunction of the proximal operator into a neural network framework \cite{hasannasab2020parseval}, we note that no separation of the dictionary and the pseudo-inverse into two independent variables to learn the weighted operator as used here.


Our main contributions are summarized below:
\begin{itemize}
\item We theoretically transform one-layer dictionary learning into transform learning and Q-Metric learning, and deduce how to convert DDL into DeTraMe-Net.
\item Such joint transform learning and Q-Metric learning are successfully and easily implemented as a tandem
of a linear layer and an RNN. A convolutional layer can be chosen for the linear part, and the RNN can also be simplified into a Convolutional-RNN. To the best of our knowledge, this is the first work which makes an insightful bridge between DDL methods and the combination of linear layers and RNNs, with the associated performance gains.
\item The transform and Q-Metric learning uses two independent variables, one for the dictionary and the other for the dual frame operator of the dictionary. This
bridges the current work
to conventional SDL while introducing more discriminative power, and allowing the use of faster learning procedures than the  original DL.
\item The Q-Metric can also be viewed as a parametric non-separable nonlinear activation function, while in current neural network architectures, very few non-separable nonlinear operators are used (softmax, max pooling, average pooling). As a component of a neural network, it can be flexibly inserted into any network architecture to easily construct a DL layer.
\item The proposed DeTraMe-Net is demonstrated to not only improve the discrimination power of DDL, but to also achieve a better performance than state-of-the-art CNNs.
\end{itemize}

The paper is organized as follows: In Section \ref{sec:related}, we introduce the required background material. We derive the theoretical basis for our novel approach in Section \ref{sec:DeTraMe-Net}.  Its algorithmic solution is investigated in Section \ref{sec:algorithm}. Substantiating experimental results and evaluations are presented in Section~\ref{sec:experiments}. Finally, we provide some concluding remarks in Section \ref{sec:conclusion}.

\subsection{Notation}
\begin{table}[!h]
    \centering
    \begin{tabular}{l|l}
    \hline
         Symbols& Descriptions  \\
    \hline
        $\mathbf{A},~(\mathbf{a}_i),~({a}_{i,j})$ & A Matrix\\
        $\mathbf{A}^\top,~\mathbf{A}^{(-1)}$& The transpose and inverse of matrices\\
        $\mathbf{I}$ & The Identity Matrix\\
        $a_{i,j}$&  The $i^{th}$ row and $j^{th}$ column element of a matrix $\mathbf{A}$\\
        $\mathbf{a}$, $a_i$ & A Vector and its $i^{th}$ element\\
        $\Ab$ & An Operator\\
    \hline
    \end{tabular}
    \label{tab:my_label}
\end{table}
\section{Preliminaries}
\label{sec:related}



\subsection{Dictionary Learning for Classification}
In task-driven dictionary learning \cite{mairal2009supervised}, the common method for one-layer dictionary learning classifier is to jointly learn the dictionary matrix $\mathbf{D}$, the sparse representation $\mathbf{a}$ of a given vector $\mathbf{x}$, and the classifier parameter $\mathbf{C}$. Let  $(\mathbf{x}_j)_{1\le j \le N}$ be the data and $(\mathbf{y}_j)_{1\le j \le N}$ the associated labels. Task-driven DL  can be expressed as finding
\begin{equation}
    \argmind{\mathbf{D,(a_j)_{1\le j \le N},C}}{\sum_{j=1}^N f(\mathbf{x}_j,\mathbf{D},\mathbf{a}_j)+g(\mathbf{x}_j,\mathbf{y}_j,\mathbf{D},\mathbf{a}_j,\mathbf{C})}.
\end{equation}
In SDL, we learn the composition of a dictionary and a sparse reconstruction in order to reconstruct or synthesize the data, hence yielding the standard formulation,
\begin{equation}
    f(\mathbf{x,D,a})=\mathbf{
    \frac{1}{2} \|x-Da\|^2+\lambda \|a\|_1}, \quad \lambda \in (0,+\infty).
\end{equation}
Alternatively, in ADL, we directly operate on the data using a dictionary, leading to,
\begin{equation}
    f(\mathbf{x,D,a})=
    \mathbf{\frac{1}{2} \|a-Dx\|^2+\lambda \|a\|_1},
    \quad \lambda \in (0,+\infty).
\end{equation}
The term $g(\mathbf{x,y,D,a,C})$ may correspond to various kinds of loss functions, such as least-squares, cross-entropy, or hinge loss.

\subsection{Deep Dictionary Learning for Classification}
An efficient DDL approach \cite{mahdizadehaghdam2019deep} consists of computing
\begin{equation} \label{equ:ddl-1}
    \mathbf{\hat{y}=\varphi (C x^{(s)})},
\end{equation}
where $\mathbf{\hat{y}}$ denotes the estimated label, $\mathbf{C}$ is the classifier matrix, $\mathbf{\varphi}$ is a nonlinear function, and
\begin{equation}\label{equ:ddl-2}
\begin{split}
    \xb^{(s)}= &\Pb^{(s)}\circ \Mb_{\Db^{(s)}} \circ \Pb^{(s-1)}\circ \Mb_{\Db^{(s-1)}}\circ \\
    &\dots \circ \Pb^{(1)}\circ \Mb_{\Db^{(1)}}(\xb^{(0)}),
\end{split}
\end{equation}
where $\circ$ denotes the composition of operators. For every layer $r \in \{1,\dots, s\},~\Pb^{(r)}$ is a reshaping operator, which is a tall matrix. Moreover, $\Mb_{\Db^{(r)}}$ is a nonlinear operator computing a sparse representation within a synthesis dictionary matrix $\Db^{(r)}$. More precisely, for a given matrix $\Db^{(r)} \in \mathbb{R}^{m_r \times k_r}$,
\begin{equation} \label{equ:ddl-main}
\begin{split}
        \Mb_{\Db^{(r)}}:&~\mathbb{R}^{m_r} \to \mathbb{R}^{k_r} \\
        &\xb \mapsto \argmind{\mathbf{a}\in \mathbb{R}^{k_r}}~\mathcal{L}^R(\Db^{(r)},\mathbf{a},\xb),
\end{split}
\end{equation}
with
\begin{multline}
\mathcal{L}^R(\Db^{(r)},\mathbf{a},\xb) = \frac{1}{2}\|\xb-\Db^{(r)}\ab\|^2_F+\lambda \psi_r(\ab)+\frac{
    \alpha}{2}\|\ab\|_2^2\\
 +(\db^{(r)})^\top \ab ,
    \label{equ:ddl-detailed-main}
\end{multline}
where $(\lambda,\alpha) \in (0,+\infty)^2$, $\db^{(r)}\in \mathbb{R}^{k_r}$, and $\psi_r$ is a function in $\Gamma_0(\mathbb{R}^{k_r})$, the class of proper lower semicontinuous convex functions from $\mathbb{R}^{k_r}$ to $(-\infty,+\infty]$. A simple choice consists in setting $\db^{(r)}$ to zero, while
adopting the following specific form for $\psi_r$;
\begin{equation}\label{equ:psi}
    \psi_r =\|\cdot\|_1+\iota_{[0,+\infty)^{k_r}},
\end{equation}
where $\iota_S$ denotes the indicator function of a set $S$ (equal to zero in $S$ and $+\infty$ otherwise). Note that Eq. (\ref{equ:ddl-main}) corresponds to the minimization of a strongly convex function, which thus admits a unique minimizer, so making the operator $\Mb_{\Db^{(r)}}$ properly defined.

\section{Deep Metric and Transform Learning}
\label{sec:DeTraMe-Net}

\subsection{Proximal interpretation}
Our goal here is to establish an equivalent but more insightful solution for $\Mb_{\Db}$ in each layer.\\
\begin{theorem}
Let $\mathcal{L}^R$ be the function defined by eq.~(\ref{equ:ddl-main}).
For every $\mathbf{D} \in \mathbb{R}^{m \times k}$,
let $\Qb=\Db^\top \Db+\alpha \Ib$, let  $\Fb=\Qb^{-1}\Db^\top$, and let $\cb=\Qb^{-1}\db$.
Then, for every $\mathbf{x}\in \mathbb{R}^m$,
\begin{equation}\label{equ:newMD}
    \mathbf{M_D(x)}=\argmind{\ab\in\mathbb{R}^k}\mathcal{L}^R(\Db,\ab,\xb)=\operatorname{prox}^{\Qb}_{\lambda \psi}(\Fb \xb-\cb),
\end{equation}
where $\operatorname{prox}^{\Qb}_{\lambda \psi}$ denotes the proximity operator of function $\lambda \psi$ in the metric
$\|\cdot \|_{\Qb} = \sqrt{(\cdot)^\top \Qb (\cdot)}$ induced by $\Qb$ \cite{Combettes_2010,Chouzenoux14jota}.
\end{theorem}
\begin{proof}
To simplify notation, we omit the superscript which denotes the layer in Eq.~(\ref{equ:ddl-main}) which, in turn, aims at finding the sparse representation $\mathbf{a}$.
For every $\mathbf{D} \in \mathbb{R}^{m \times k},~\mathbf{a}\in \mathbb{R}^k$, and $\mathbf{x}\in \mathbb{R}^m$, Eq. (\ref{equ:ddl-detailed-main}) can thus be re-expressed as follows:
\begin{equation}
   \begin{split}
    \mathcal{L}^R(\Db,\ab,\xb) &=\frac{
    1}{2}\big(\|\xb\|^2-2 \xb^\top \Db \ab +\ab^\top(\Db^\top \Db+\alpha \mathbf{I} )\ab \big) \\
    &+\lambda \psi(\ab)+\db^\top \ab \\
    & = \tilde{\mathcal{L}}^R(\Db,\ab,\xb)+\frac{1}{2}(\|\xb\|^2-\|\Fb \xb\|_{\Qb}^2-\|\cb\|_{\Qb}^2)\\
    &+\xb^\top \Db \cb,\\
\end{split}
\end{equation}
where
\begin{equation}\label{equ:proximal-a}
\tilde{\mathcal{L}}^R(\Db,\ab,\xb)=\frac{1}{2}\|\ab-\Fb \xb+\cb\|_{\Qb}^2+\lambda \psi(\ab),
\end{equation}
with
\begin{equation}\label{equ:transform}
\Qb=\Db^\top \Db+\alpha \Ib,\, \Fb=\Qb^{-1}\Db^\top,~\cb=\Qb^{-1}\db,
\end{equation}
and $\|\cdot \|_{\Qb} = \sqrt{(\cdot)^\top \Qb (\cdot)}$ denotes the weighted Euclidean norm induced by $\Qb$. Determining the optimal sparse representation
$\mathbf{a}$ of
$\mathbf{x}\in \mathbb{R}^m$ is therefore, equivalent to computing the proximity operator in Eq. (\ref{equ:proximal-a}), that is Eq. (\ref{equ:newMD-claim}):
\begin{equation}\label{equ:newMD}
    \mathbf{\Mb_D(x)}=\argmind{\ab\in\mathbb{R}^k}\tilde{\mathcal{L}}^R(\Db,\ab,\xb)=\operatorname{prox}^{\Qb}_{\lambda \psi}(\Fb \xb-\cb).
\end{equation}
\end{proof}
This thus establishes a re-expression of the solution of the representation procedure as
the proximity operator of $\lambda\psi$ within the metric induced by the symmetric definite positive matrix $\mathbf{Q}$ \cite{Combettes_2010,Chouzenoux14jota}. Furthermore, it shows that the SDL can be equivalently viewed as an ADL formulation involving the dictionary matrix $\mathbf{F}$, provided that a proper metric is chosen.

\subsection{Multilayer representation}
Consequently, by substituting Eq. (\ref{equ:newMD}) in Eqs. (\ref{equ:ddl-1}) and (\ref{equ:ddl-2}), the DDL model can be re-expressed in a more concise and comprehensive form as
\begin{equation}\label{equ:JCDML}
\begin{split}
    \mathbf{\hat{y}=}& \varphi \circ \Ab^{(s+1)}\circ \operatorname{prox}^{\Qb^{(s)}}_{\lambda \psi_s} \circ \Ab^{(s)} \circ \operatorname{prox}^{\Qb^{(s-1)}}_{\lambda \psi_{s-1}} \circ \dots\\
    &\circ \operatorname{prox}^{\Qb^{(1)}}_{\lambda \psi_{1}} \circ \Ab^{(1)}(\xb^{(0)}),
\end{split}
\end{equation}
where, for $1\leq r \leq s$, the affine operators $\Ab^{(r)}$ 
mapping $\zb^{(r-1)}\in \mathbb{R}^{k_{r-1}}$ to $\zb^{(r)}\in \mathbb{R}^{k_{r}}$ by an analysis transform $\Wb^{(r)}$ and a shift term $\cb^{(r)}$, and explicitly as,
\begin{equation}\label{equ:transformed-ddl}
\begin{split}
    \forall r\in \{1,\dots,s\}, \Ab^{(r)}:& \mathbb{R}^{k_{r-1}} \to \mathbb{R}^{k_r}\\
    & \zb^{(r-1)} \mapsto \Wb^{(r)}\zb^{(r)}-\cb^{(r)}
\end{split}
\end{equation}
with $k_0=m_1$ and
\begin{equation}
    \begin{split}
        &\Wb^{(1)}=\Fb^{(1)},\\
        \forall r\in \{2,\dots,s\},\\
        &\Wb^{(r)}=\Fb^{(r)}\Pb^{(r-1)},\\
        &\Wb^{(s+1)}=\mathbf{C} \Pb^{(s)}\\
        \forall r \in \{1,\dots,s\},\\
        &\Qb^{(r)}=(\Db^{(r)})^\top \Db^{(r)}+\alpha \Ib,\\
        &\Fb^{(r)}=(\Qb^{(r)})^{-1}(\Db^{(r)})^\top,\\
        &\cb^{(r)}=(\Qb^{(r)})^{-1}\db^{(r)}.\\
    \end{split}
\end{equation}


\noindent Eq. (\ref{equ:transformed-ddl}) shows that, for each layer $r$, we obtain a structure similar to a linear layer by treating $\Wb^{(r)}$ as the weight operator and $\cb^{(r)}$ as the bias parameter, which are referred as the Transform learning part in DeTraMe method.
In standard Forward Neural Networks (FNNs), the activation functions can be interpreted as proximity operators of convex functions
\cite{Combettes2018}. Eq. \eqref{equ:JCDML} attests that our model is more general, in the sense that different metrics are introduced for these operators. In the next section, we propose an efficient method to learn these metrics in a supervised manner.


\section{Q-Metric Learning}
\label{sec:algorithm}

\subsection{Prox computation}
Reformulation \eqref{equ:JCDML} has the great advantage to allow us to benefit from algorithmic frameworks developed for FNNs, provided that we are able to compute efficiently
\begin{equation}\label{equ:proximal-operator}
    \operatorname{prox}^{\Qb}_{\lambda \psi}(\Zb)=\argmind{\Ub \in \mathbb{R}^{k\times N}} \frac{
    1}{2}\|\Ub-\Zb\|_{F,\Qb}^2+\lambda \psi(\Ub),
\end{equation}
where $\|\cdot\|_{F,\Qb}=\sqrt{\operatorname{tr}((\cdot)\Qb(\cdot)^\top)}$ is the $\Qb$-weighted Frobenius norm.
Hereabove, $\Zb$ is a matrix where the $N$ samples associated with the training set have been stacked columnwise. A similar convention is used to construct $\mathbf{X}$ and $\mathbf{Y}$ from $(\mathbf{x}_j)_{1\le j \le N}$ and $(\mathbf{y}_j)_{1\le j \le N}$.
\begin{theorem}\label{th:2}
Assume that an elastic-net like regularization is adopted by setting
$\psi=\|\cdot\|_1+\iota_{[0,+\infty)^{k\times N}}+\frac{\beta}{2\lambda}\|\cdot\|_{F}^2$ with $\beta \in (0,+\infty)$. For every $\Zb\in \mathbb{R}^{k\times N}$, the elements of $\operatorname{prox}^{\Qb}_{\lambda \psi}(\Zb)$ in
eq. \eqref{equ:proximal-operator}
satisfy
for every $i\in \{1,\ldots,k\}$,
and $j\in \{1,\ldots,N\}$,
\begin{equation}\label{e:fixptc}
u_{i,j}=
 \begin{cases}
  \frac{q_{i,i}}{q_{i,i}+\beta}z_{i,j} - v_{i,j} & \mbox{if $q_{i,i} z_{i,j}> (q_{i,i}+\beta)v_{i,j}$}\\
  0 &\mbox{otherwise},
  \end{cases}
\end{equation}
where $v_{i,j}=\frac{\lambda+\sum_{\ell=1,\ell \neq i }^k q_{i,\ell}(u_{\ell,j}-z_{\ell,j})}{q_{i,i}+\beta}$.
\end{theorem}
\begin{proof}
As $\Qb=\Db^\top \Db+\alpha \Ib$ in Eq. (\ref{equ:transform}), Eq. (\ref{equ:proximal-operator}) is actually equivalent to solving the following optimization problem:
\begin{equation}\label{equ:final-prox}
\begin{split}
    \minimize{\Ub \in [0,+\infty)^{k\times N}}{}& \frac{1}{2}\|\Db(\Ub-\Zb)\|_F^2+\frac{\alpha}{2}\|\Ub-\Zb\|_F^2\\ &+\frac{\beta}{2}\|\Ub\|_F^2+\lambda\|\Ub\|_1.
\end{split}
\end{equation}\\

\noindent \textit{\textbf{Claim:}} \textit{We show next that the solution of Eq. (\ref{equ:final-prox}) is obtained as an iteration of the form:}
\begin{equation}\label{equ:u-update}
\Ub_{t+1}=
  \operatorname{ReLU}\big((\mathbf{h} \mathbf{1}^\top)\odot \Zb+\widetilde{\Wb}(\Ub_t-\Zb)-\mathbf{b} \mathbf{1}^\top\big).
\end{equation}
\vspace{3mm}

Various iterative splitting methods could be used to find the unique minimizer of the above optimized convex function \cite{boyd2004convex,Komodakis2014}. Our purpose is to develop an algorithmic solution for which classical NN learning techniques can be applied in a fast and convenient manner.
By subdifferential calculus, the solution $\Ub$ to the problem \eqref{equ:final-prox} satisfies the following optimality condition:
\begin{equation}\label{equ:extreme}
    \mathbf{0}\in \Qb(\Ub-\Zb)+\beta \Ub+\lambda \partial\widetilde{\psi}(\Ub),
\end{equation}
where $\widetilde{\psi} = \|\cdot\|_1+\iota_{[0,+\infty)^{k\times N}}$.
Element-wise rewriting of Eq.~(\ref{equ:extreme}) yields, for every $i\in \{1,\ldots,k\}$,
and $j\in \{1,\ldots,N\}$,
\begin{equation}\label{e:optcond}
\begin{split}
    0\in\; \sum_{\ell=1}^k q_{i,\ell}(u_{\ell,j}-z_{\ell,j})&+ \beta u_{i,j}
   +
 \begin{cases}
    (-\infty,\lambda ]  & \quad \mbox{if $u_{i,j}=0$}\\
    \lambda & \quad \mbox{if $u_{i,j}>0$}\\
    \varnothing & \quad \mbox{if $u_{i,j}< 0$}
  \end{cases}.
\end{split}
\end{equation}

Let us adopt a block-coordinate approach and update the $i$-th
row of $\mathbf{U}$ by fixing all the other ones.
As $\mathbf{Q}$ is a positive definite matrix, $q_{i,i} > 0$ and Eq.~\eqref{e:optcond}
implies that
\begin{equation}\label{e:fixptc}
u_{i,j}=
 \begin{cases}
  \frac{q_{i,i}}{q_{i,i}+\beta}z_{i,j} - v_{i,j} & \mbox{if $q_{i,i} z_{i,j}> (q_{i,i}+\beta)v_{i,j}$}\\
  0 &\mbox{otherwise},
  \end{cases}
\end{equation}
where $v_{i,j}=\frac{\lambda+\sum_{\ell=1,\ell \neq i }^k q_{i,\ell}(u_{\ell,j}-z_{\ell,j})}{q_{i,i}+\beta}$.
\end{proof}
\noindent Let
\begin{equation}\label{e:reparam}
\begin{split}
    &\widetilde{\mathbf{W}}= -
    \left(\frac{q_{i,\ell}}{q_{i,i}+\beta}\delta_{i-\ell}\right)_{1\leq i,\ell \leq k},\\
    &\mathbf{h}=\left(\frac{q_{i,i}}{q_{i,i}+\beta}\right)_{1\leq i \leq k}\in [0,1]^k,\\
    &\mathbf{b}=\left(\frac{\lambda}{q_{i,i}+\beta}\right)_{1\leq i \leq k}\in [0,+\infty)^k,\\
    & \mathbf{1}=[1,\ldots,1]^\top\in \mathbb{R}^N,\\
\end{split}
\end{equation}
where $(\delta_\ell)_{\ell \in \mathbb{Z}}$ is the Kronecker sequence
(equal to 1 when $\ell=0$ and 0 otherwise).
Then, Eq. \eqref{e:fixptc} suggests that the elements of
$\Ub$ can be globally updated, at iteration $t$,
{as shown in Eq. (\ref{equ:u-update}): }
\begin{equation*}
\Ub_{t+1}=
  \operatorname{ReLU}\big((\mathbf{h} \mathbf{1}^\top)\odot \Zb+\widetilde{\Wb}(\Ub_t-\Zb)-\mathbf{b} \mathbf{1}^\top\big),
\end{equation*}
with $\odot$ denoting the Hadamard (element-wise) product.
Note that a similar expression can be derived by applying a preconditioned
forward-backward algorithm \cite{Chouzenoux14jota} to Eq. \eqref{equ:final-prox}, where the
preconditioning matrix is $\operatorname{Diag}(q_{1,1},\ldots,q_{k,k})$, which has been detailed in the Appendix \ref{sec:appendix-A}.
The implementation of the method allowing us to compute the proximity
operator in (\ref{equ:proximal-operator}) is summarized below:
\renewcommand{\algorithmicrequire}{\textbf{Input:}}
\renewcommand{\algorithmicensure}{\textbf{Output:}}
\begin{algorithm}[htb]
 \caption{Q-Metric ReLU Computation }
 \label{alg:alg1}
	\begin{algorithmic}[1]
		\REQUIRE matrix $\widetilde{\mathbf{W}}$ and $\Zb$,
		vectors $\mathbf{h}$ and $\mathbf{b}$, and maximum iteration number $t_{\rm max}$
		\ENSURE Sparse Representation $\mathbf{U^*}$
		\STATE Initialize $\mathbf{U}_0$ as the null matrix and set $t=0$
		\WHILE {not converged \text{and} $t< t_{\rm max}$}
		\STATE
		Update $\mathbf{U}_{t+1}$ according to Eq. (\ref{equ:u-update})
		\STATE $t\leftarrow t+1$
		\ENDWHILE
	\end{algorithmic}
\end{algorithm}

\subsection{RNN implementation}
Given $\widetilde{\mathbf{W}}$, $\mathbf{h}$, and $\mathbf{b}$,
Alg. (\ref{alg:alg1}) can be viewed as an RNN structure
for which $\Ub_t$ is the hidden variable and
$\Zb$ is a constant input over time. By taking advantage
of existing gradient back-propagation techniques for RNNs,
$(\widetilde{\mathbf{W}},\mathbf{h},\mathbf{b})$ can thus be directly computed
in order to minimize the global loss $\mathcal{L}$. This shows that, thanks to the
re-parameterization in Eq. \eqref{e:reparam},
Q-Metric Learning has been recast as the training of a specific RNN.


Note that $\Qb$
is a $k\times k$ symmetric matrix.
In order to reduce the number of parameters and ease of optimizing them,
we choose a block-diagonal structure for $\mathbf{Q}$.
In addition, for each of the blocks, either an arbitrary or
convolutive structure can be adopted.
Since the structure of $\mathbf{Q}$ is reflected by the structure
of $\widetilde{\Wb}$, this leads in Eq.  \eqref{equ:u-update} to fully connected or convolutional
layers where the channel outputs are linked to non overlapping blocks of the inputs.
In our experiments on images, Convolutional-RNNs have been preferred for practical efficiency.


\subsection{Training procedure}
We have finally transformed our DDL approach in an alternation of
linear layers and specific RNNs. This not only simplifies the
implementation of the resulting DeTraMe-Net
by making use of standard NN tools,
but also allows us to employ well-established stochastic gradient-based learning strategies.
Let $\rho_t>0$ be the learning rate at iteration $t$, the simplified
form of a training method for DeTraMe-Nets is provided in Alg.~\ref{alg:alg2}.

\renewcommand{\algorithmicrequire}{\textbf{Initialization:}}
\renewcommand{\algorithmicensure}{\textbf{Output:}}
\begin{algorithm}[!htb]
	\caption{Deep Transform and Metric Learning Network}
	\label{alg:alg2}
	\begin{algorithmic}[1]
	    \REQUIRE
	    \FOR{$r=1,\dots,s+1$}
	    \STATE Randomly initialize $\mathbf{W}^{(r)}_0$, $\mathbf{c}^{(r)}_0$, $\widetilde{\mathbf{W}}_{0}^{(r)}$, $\mathbf{h}^{(r)}_0$,  and $\mathbf{b}^{(r)}_0$.
	    \ENDFOR
	    \STATE Set $t=0$.
		\WHILE {not converged \text{and} $t< t_{\rm max}$}
		\STATE \textbf{Forward pass:}
		\STATE $\mathbf{U}_t^{(0)}=\mathbf{X}$
		\FOR{$r=1,\dots,s+1$}
		\STATE $\mathbf{Z}_t^{(r)}=\mathbf{W}_t^{(r)}\mathbf{U}_t^{(r-1)}-\mathbf{c}_t^{(r)}$
		\IF{$r\leq s$}
		\STATE $\mathbf{U}_t^{(r)} =\operatorname{prox}^{\Qb_t^{(r)}}_{\lambda \psi_r}(\mathbf{Z}_t^{(r)})$ by Alg. \ref{alg:alg1}
		\ENDIF
		\ENDFOR
		\STATE $\mathbf{\hat{Y}}_t = \varphi(\mathbf{Z}_t^{(s+1)})$
		\STATE \textbf{Loss:} \mbox{$\mathcal{L}'(\boldsymbol{\theta}_t) = \mathcal{L}(\mathbf{Y},\mathbf{\hat{Y}}_t)$, $\boldsymbol{\theta}_t$: vector of all parameters}
		\STATE \textbf{Backward pass:}
		\FOR{$r=1,\dots,s+1$}
		\STATE $\mathbf{W}_{t+1}^{(r)}=\mathbf{W}_t^{(r)}-\rho_t \frac{\partial \mathcal{L}'}{\partial \mathbf{W}^{(r)}}(\boldsymbol{\theta}_t)$
		\STATE$\mathbf{c}_{t+1}^{(r)}=\mathbf{c}_t^{(r)}-\rho_t \frac{\partial \mathcal{L}'}{\partial \mathbf{c}^{(r)}}(\boldsymbol{\theta}_t)$
		\ENDFOR
		\FOR{$r=1,\dots,s$}
		\STATE $\widetilde{\mathbf{W}}_{t+1}^{(r)}=\mathsf{P}_{\mathcal{D}_0}\left(\widetilde{\mathbf{W}}^{(r)}_t-\rho_t \frac{\partial \mathcal{L}'}{\partial \widetilde{\mathbf{W}}^{(r)}}
		(\boldsymbol{\theta}_t)\right)
		$
		\STATE $\mathbf{h}_{t+1}^{(r)}=\mathsf{P}_{[0,1]^k}\left(\mathbf{h}_t^{(r)}-\rho_t \frac{\partial \mathcal{L}'}{\partial \mathbf{h}^{(r)}}(\boldsymbol{\theta}_t)\right)$
		\STATE$\mathbf{b}_{t+1}^{(r)}=\mathsf{P}_{[0,+\infty)^k}\left(\mathbf{b}_t^{(r)}-\rho_t \frac{\partial \mathcal{L}'}{\partial \mathbf{b}^{(r)}}(\boldsymbol{\theta}_t)\right)$
		\ENDFOR
		\STATE $t\leftarrow t+1$
		\ENDWHILE
	\end{algorithmic}
\end{algorithm}
The constraints on the parameters of the RNNs have been imposed by projections.
In Alg.~\ref{alg:alg2}, $\mathsf{P}_{\mathcal{S}}$ denotes the projection onto a nonempty closed convex set
$\mathcal{S}$ and
$\mathcal{D}_0$ is the vector space of $k\times k$ matrices with diagonal terms equal to 0.

\section{Experiments and Results}
\label{sec:experiments}
In this section, our DeTraMe-Net method is evaluated on three popular datasets, namely CIFAR10 \cite{CIFAR}, CIFAR100 \cite{CIFAR} and Street View House Numbers (SVHN) \cite{SVHN}. Since the common NN architectures are plain networks such as ALL-CNN \cite{ALLCNN} and residual ones, such as ResNet \cite{resnet} and WideResNet \cite{WRN}, we compare DeTraMe-Net with these three respective state-of-the-art architectures. All the experiments of the state-of-the-arts and our method are re-implemented and repeated over 5 runs.

\begin{table*}[!htb]
    \centering
    \resizebox{\textwidth}{!}{%
    \begin{tabular}{c|c|c|c|c}
    \hline
    \hline
        DeTraMe-PlainNet 3-layer &PlainNet 3-layer & PlainNet 6-layer &  PlainNet 9-layer &  PlainNet 12-layer \\
        \hline
         \multicolumn{5}{c}{Input 32 x 32 RGB Image with dropout(0.2) }\\
        \hline
        $3 \times 3$ conv 96 &$3 \times 3$ conv 96 RELU &$3 \times 3$ conv 96 RELU&$3 \times 3$ conv 96 RELU & $3 \times 3$ conv 96 RELU \\
       + Q-Metric: $3 \times 3$ conv 96  & &$3 \times 3$ conv 96 RELU&$3 \times 3$ conv 96 RELU &$3 \times 3$ conv 96 RELU  \\
        &                   &                  &$3 \times 3$ conv 96 RELU &$3 \times 3$ conv 96 RELU  \\
        &                &                  &with stride=2, dropout(0.5)  &with stride=2, dropout(0.5)  \\

        &                &  & &$3 \times 3$ conv 192 RELU  \\
        \hline
        $3 \times 3$ conv 96 &$3 \times 3$ conv 96 RELU  &$3 \times 3$ conv 96 RELU& $3 \times 3$ conv 192 RELU& $3 \times 3$ conv 192 RELU  \\
        with stride=2 &with stride=2       &  with stride=2, dropout(0.5)& $3 \times 3$ conv 192 RELU& $3 \times 3$ conv 192 RELU  \\
        + Q-Metric: $3 \times 3$ conv 96                    &  &$3 \times 3$ conv 192 RELU & $3 \times 3$ conv 192 RELU& with stride=2, dropout(0.5) \\
        &                    &                  &with stride=2, dropout(0.5)  & $3 \times 3$ conv 192 RELU \\
        \hline
        $3 \times 3$ conv 10 &$3 \times 3$ conv 10 RELU &  $3 \times 3$ conv 192 RELU  & $3 \times 3$ conv 192 RELU& $3 \times 3$ conv 192 RELU \\
         with stride=2 &  with stride=2 &$3 \times 3$ conv 10 RELU& $1 \times 1$ conv 192 RELU& with stride=2 \\
        +Q-Metric: $3 \times 3$ conv 10 &             &with stride=2 & $1 \times 1$ conv 10 RELU& $3 \times 3$ conv 192 RELU \\
        & & & & $1 \times 1$ conv 192 RELU \\
        &  & & & $1 \times 1$ conv 10 RELU \\
        \hline
       \multicolumn{5}{c}{Global Average Pooling} \\
       \hline
       \multicolumn{5}{c}{Softmax} \\
       \hline
       \hline
    \end{tabular}}
    \caption{Model Description of PlainNet}
    \label{tab:PlainNetModel}
\end{table*}

\subsection{Architectures}
Since we break SDL into two independent linear layer and RNN parts, RNNs can be flexibly inserted into any nonlinear layer of a deep neural network. After choosing convolutional linear layers, we can construct two different architectures when inserting RNN into Plain Networks and residual blocks.

One is to replace all the RELU activation layers in PlainNet with Q-Metric ReLU, leading to DeTraMe-PlainNet.
Another is to replace the RELU layer inside the block in ResNet by Q-Metric ReLU, giving rise to DeTraMe-ResNet. When replacing all the RELU layers, DeTraMe-PlainNet becomes equivalent to DDL as explained in Section \ref{sec:algorithm}. When only replacing a single RELU layer in the ResNet architecture, a new DeTraMe-ResNet structure is built. The detailed architectures  are illustrated in the Appendix \ref{sec:appendix-B}.


For the PlainNet, we use a 9 layer architecture similar
to ALL-CNN \cite{ALLCNN} with dropouts, as listed in Table \ref{tab:PlainNetModel}. For the ResNet architecture, we follow the setting in \cite{resnet}, the first layer is a $3 \times 3$ convolutional layer with 16 filters. 3 residual blocks with output map size of 32, 16, and 8 are then used with 16, 32 and 64 filters for each block. The network ends up with a global average pooling and a fully-connected layer.
The parameters listed in Table \ref{tab:resnet}
are respectively chosen equal to $n=1,3,9,18,27$ for ResNet 8, 20, 56, 110 and 164-layer networks, and we respectively use $n=2, q=4$ and $n=2, q=8$ for WideResNet 16-4 and WideResNet 16-8 networks as suggested in \cite{WRN}.

\begin{table}[htb]
    \centering
    \begin{tabular}{c|c|c|c}
    \hline
    \hline
         output map size& $32 \times 32$ & $16 \times 16$ & $8 \times 8$  \\
         \hline
         \# layers & $1+2n$ & $2n$ &  $2n$\\
         \#filters & $16$ &  $32$ & $64$\\
         \hline
         WideResNet \#filters  & $16 \times q$ &  $32  \times q$ & $64 \times q$\\
         \hline
    \hline
    \end{tabular}
    \caption{ResNet Model \cite{resnet}}
    \label{tab:resnet}
\end{table}

For DeTraMe-Net,
we use convolutional RNNs having the same filter size (resp. number of channels) as those in the convolutional layer before.
The number of parameters of each model as well as the number of iterations performed in RNNs, are indicated in Table \ref{tab:results}.

\subsection{Datasets and Training Settings}
\text{\textbf{CIFAR10}} \cite{CIFAR}
contains 60,000 $32 \times 32$ color images divided into 10 classes. 50,000 images are used for training and 10,000 images  for testing.
\textbf{CIFAR100} \cite{CIFAR} is also constituted of $32\times32$ color images. However, it includes 100 classes with 50,000 images for training and 10,000  images for testing.
\textbf{SVHN} \cite{SVHN}
contains 630,420 color images with size $32 \times 32$. 604,388 images are used for training and 26,032 images are used for testing.

For CIFAR datasets, the normalized input image is $32 \times 32$ randomly cropped after $4\times 4$ padding on each sides
of the image and random flipping, similarly to \cite{resnet,WRN}. No other data augmentation is used. For SVHN, we normalize the range of the images between 0 and 1. All the models are trained on an Nvidia V100 32Gb GPU with 128 mini-batch size.
The models of both PlainNet and ResNet architectures are trained by SGD optimizer with momentum equal to 0.9 and a weight decay of $5\times 10^{-4}$. On CIFAR datasets, the algorithm starts with a learning rate of 0.1. 200 epochs are used to train the models, and the learning rate is reduced by 0.2 at the 60-th, 120-th, 160-th and 200-th epochs. On SVHN dataset, a learning rate of 0.01 is used at the beginning and is then divided by 10 at the 80-th and 120-th epochs within a total of 160 epochs. The same settings are used as in \cite{WRN}.

\subsection{Results}
\subsubsection{DeTraMe-Net vs. DDL}
First, we compare our results with those achieved by the DDL approach in \cite{mahdizadehaghdam2019deep}, as both DeTraMe-Net and DDL with 9-layer follow the ALL-CNN architecture in \cite{ALLCNN}.
\begin{table}[!htb]
    \centering
    \resizebox{0.95\columnwidth}{!}{
    \begin{tabular}{|c|c|c|c|}
    \hline
         Model & \# Parameters &CIFAR10 &CIFAR100  \\
    \hline
        PlainNet 9-layer \cite{ALLCNN} &1.4M  &90.31\% $\pm$ 0.31\% &66.15\% $\pm$ 0.61\% \\
         DDL 9 \cite{mahdizadehaghdam2019deep}&1.4M& 93.04\%$^*$ & 68.76\%$^*$ \\
        \hline
         DeTraMe-Net 9&3.0M&\textbf{93.05\%} $\pm$ 0.46\%  &\textbf{69.68\%} $\pm$ 0.50\%  \\
         DeTraMe-Net 9 (Best)&3.0M&\textbf{93.40\%} & \textbf{70.34\%}\\
    \hline
    \end{tabular}
    }
    \caption{
    {Accuracy: DeTraMe-Net vs. DDL: the architectures are listed in the fourth column in Table \ref{tab:PlainNetModel}. The number with '*' was reported in the original paper.
    }}
    \label{tab:ddlcompare}
\end{table}

\begin{table*}[!htb]
    \centering
    \begin{tabular}{c|c|c|c|c|c|c}
    \hline
    \hline
    Accuracy (\%)  &\multicolumn{2}{|c|}{CIFAR10 +} &\multicolumn{2}{|c}{CIFAR100 +}&\multicolumn{2}{|c}{SVHN}\\
    \hline
         Network Architectures   & Original  & DeTraMe-Net & Original  & DeTraMe-Net & Original  & DeTraMe-Net\\
         &  & (\#iteration) &    & (\#iteration) &    & (\#iteration) \\
         \hline
         PlainNet 3-layer & 35.14 $\pm$ 4.94 & \textbf{88.51} $\pm$ 0.17 (5)&22.01 $\pm$ 1.24 & \textbf{64.99} $\pm$ 0.34 (3)&45.64&\textbf{97.21} (8)\\
         PlainNet 6-layer& 86.71 $\pm$ 0.36 & \textbf{92.24} $\pm$ 0.32 (2)&62.81 $\pm$ 0.75&\textbf{69.49} $\pm$ 0.61 (2)&97.55&\textbf{98.17} (5)\\
         PlainNet 9-layer  &90.31 $\pm$ 0.31 & \textbf{93.05} $\pm$ 0.46 (2)&66.15 $\pm$ 0.61 &\textbf{69.68} $\pm$ 0.50 (2)& 97.98&\textbf{98.26} (5)\\
         PlainNet 12-layer &91.28 $\pm$ 0.27 & \textbf{92.03} $\pm$ 0.54 (2)&68.70 $\pm$ 0.65 &\textbf{70.92} $\pm$ 0.78 (2)&98.14&\textbf{98.27} (3)\\
         \hline
         ResNet 8 &87.36 $\pm$ 0.34 & \textbf{89.13} $\pm$ 0.23 (3)&60.38 $\pm$ 0.49 &\textbf{64.50} $\pm$ 0.54 (2)&96.70&\textbf{97.50} (3)\\
         ResNet 20 &92.17 $\pm$ 0.15 & \textbf{92.19} $\pm$ 0.30 (3)&68.42 $\pm$ 0.29 &\textbf{68.62} $\pm$ 0.27 (2)&97.70&\textbf{97.82} (2)\\
         ResNet 56 &93.48 $\pm$ 0.16 & \textbf{93.54} $\pm$ 0.30 (3)&\textbf{71.52} $\pm$ 0.34 &\textbf{71.52} $\pm$ 0.44 (2)&97.96&\textbf{98.04} (2)\\
         ResNet 110 &93.57 $\pm$ 0.14& \textbf{93.68} $\pm$ 0.32 (2)&72.99 $\pm$ 0.43 &\textbf{73.05} $\pm$ 0.40 (2)&-&-\\
         \hline
         WideResNet 16-4 & \textbf{95.18} $\pm$ 0.10 & \textbf{95.18} $\pm$ 0.13 (2)& 76.72 $\pm$ 0.13& \textbf{76.85} $\pm$ 0.48 (3) &98.06&\textbf{98.16} (3)\\
         WideResNet 16-8 & 95.62 $\pm$ 0.12 & \textbf{95.66} $\pm$ 0.22 (2)& 79.55 $\pm$ 0.12 &\textbf{79.69} $\pm$ 0.55 (3)&98.17&\textbf{98.23} (3)\\
         \hline
        \hline
    \end{tabular}
    \caption{
    {
    CIFAR10 and CIFAR100 with + is trained with simple translation and flipping data augmentation. All the presented results are re-implemented and run by using the same settings. SVHN is too large to train, so it is only run once for reference.
    }
    }
    \label{tab:results}
\end{table*}

\begin{table}[!htb]
    \centering
    \begin{tabular}{|c|c|c|c|}
    \hline
         Model& \#Parameters  & Training (s)   &Testing (s)  \\
    \hline
         DDL \cite{mahdizadehaghdam2019deep}& {0.35 M}&  0.2784$^*$ & 9.40$\times 10^{{-2} ^*}$ \\
        \hline
         DeTraMe-Net 12&\textbf{2.4 M}&\textbf{0.1605} &\textbf{3.52}$\times 10^{-4} $  \\

    \hline
    \end{tabular}
    \caption{
    {Time Complexity: DeTraMe-Net vs. DDL: The number with $'*'$ was averaged based on the reported one in the original paper.
    }}
    \label{tab:ddlTimecompare1}
\end{table}

Since we break the dictionary and its pseudo inverse into two independent variables, a higher number of parameters is involved in DeTraMe-Net than in \cite{mahdizadehaghdam2019deep}. However, DeTraMe-Net presents two main advantages:

\textbf{(1)} The first one is a better capability to
discriminate: in comparsion to DDL in Table \ref{tab:ddlcompare}, the best DeTraMe-Net accuracy respectively achieves $0.36\%$ and $1.58\%$ improvements on CIFAR10 and CIFAR100 datasets and, in terms of averaged performance, $0.01\%$ and $0.92\%$ accuracy improvements are respectively obtained on these two datasets.

\textbf{(2)} The second advantage is that DeTraMe-Net is implemented in a network framework, with no need for extra functions to compute gradients at each layer, which greatly reduces the time costs. As shown in Table~\ref{tab:ddlTimecompare1}, DDL \cite{mahdizadehaghdam2019deep} processes a 28 $\times $ 28 image with 0.35M parameters in 0.2784 second for training and 9.4$\times 10^{-2}$ s for testing, while the proposed DeTraMe-Net processes a 32 $\times $ 32 image with 2.4M parameters in 0.1605 second for training and 3.52$\times 10^{-4}$ s for testing. This shows that our method with 6 times more parameters than DDL only requires half training time and a faster testing time by a factor 100. Moreover, by taking advantage of the developed implementation frameworks for neural networks, DeTraMe-Net can use up to 110 layers, while the maximum number of layers in \cite{mahdizadehaghdam2019deep} is 23.
\subsubsection{DeTraMe-Net vs. Generic CNNs}

We next compare DeTraMe-Net with generic CNNs with respect to three different aspects: \underline{Accuracy}, \underline{Parameter$\phantom{y}$number}, \underline{Capacity},  \underline{Adversarial robustness and$\phantom{y}$robustness to random noise} and \underline{Time$\phantom{y}$complexity}.

\textbf{Accuracy.} As shown in Table \ref{tab:results}, 
with the same architecture, using DeTraMe-Net structures achieves an overall better performance than all various generic CNN models do. For PlainNet architecture, DeTraMe-Net increases the accuracy with a median of $3.99\%$ on CIFAR10, $5.11\%$ on CIFAR100 and $0.45\%$ on SVHN, and respectively increases the accuracy of at least $0.75\%,~2.22\%,~0.13\%$ on theses three datasets. For ResNet architecture, DeTraMe-Net also consistently increases the accuracy with a median of $0.05\%$ on CIFAR10, $0.13\%$ on CIFAR100 and $0.10\%$ on SVHN.

\begin{figure}[!htb]
     \centering
     \includegraphics[width=0.9\columnwidth]{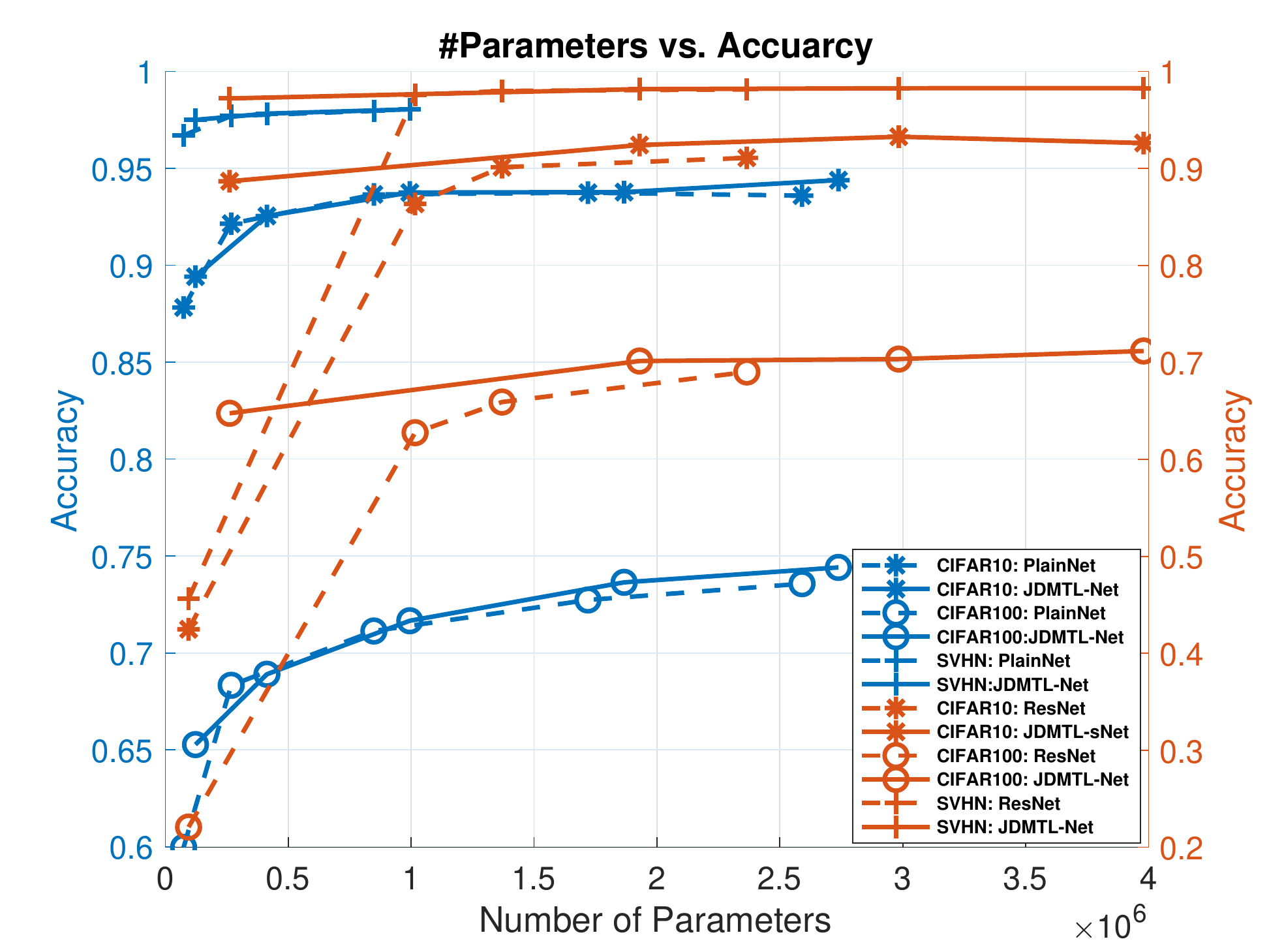}
     \caption{
     {Classification accuracy versus number of parameters. The blue color curves are based on ResNet architecture (\textit{left axis}), while the orange curves are based on PlainNet architecture (\textit{right axis}). The solid line denotes DeTraMe-Net, while the dash-line denotes the original CNNs. '*' denotes for CIFAR10, 'o' denotes for CIFAR100 and '+' denotes for SVHN.}}
     \label{fig:accvsparams}
 \end{figure}
\textbf{Parameter number.} Although, for a given architecture, DeTraMe-Net improves the accuracy, it involves more parameters. However, as demonstrated in Figure \ref{fig:accvsparams},
for a given number of parameters, DeTraMe-Net outperforms the original CNNs over all three datasets. Plots corresponding to DeTraMe-Net for both PlainNet and ResNet architectures are indeed above those associated with standard CNNs.

\begin{table*}[!htb]
    \centering
    \begin{tabular}{c|c|c|c}
    \hline
    \hline
        Model & CIFAR10+ ($\xi= 5e-2$)& CIFAR100+ ($\xi= 5e-2$)&SVHN($\xi= 2e-4$) \\
    \hline
        Original WideResNet-16-8 &$18.03\% \pm 1.75\%$&$59.61\% \pm 2.72\%$&46.84\% $\pm$ 3.05\%\\
        DeTraMe WideResNet-16-8 &\textbf{6.78\%} $\pm$ 0.73\% &\textbf{23.65\%} $\pm$ 1.84\%& \textbf{23.61\%} $\pm$ 5.61\%\\
    \hline
    \hline
    \end{tabular}
     \caption{
     \small{Fooling Rate versus Adversarial Attack. $\xi$ in \cite{UAP} controls the attack magnitude.}}
    \label{tab:UAP}
\end{table*}

\textbf{Capacity.}
In terms of \textit{depth}, comparing improvements with PlainNet and ResNet,
shows that the shallower the network, the more accurate. It is remarkable that DeTraMe-Net leads to more than $42\%$ accuracy increase for PlainNet 3-layer
on CIFAR10, CIFAR100 and SVHN datasets.
When the networks become deeper, they better capture discriminative features of the classes,
and albeit with smaller gains,
DeTraMe-Net still achieves a better accuracy than a generic deep CNN, e.g. around $0.11\%$ and $0.05\%$  higher than ResNet 110 on CIFAR10 and CIFAR100.
In terms of \textit{width}, we use WideResNet-16-4 and WideResNet-16-8 as two reference models, since both of them include 16 layers but have different widths. Table \ref{tab:results} shows that
increasing width is beneficial to DeTraMe-Net. Since the original models have already achieved excellent performance for CIFAR10, CIFAR100 and SVHN, DeTraMe-Nets with various widths show similarly slightly improved accuracies. However, the experiments still demonstrate that enlarging the width for DeTraMe-Net leads to an increase in the accuracy gain.

\textbf{Adversarial robustness and robustness to random noise.} The UAP tool \cite{UAP} is used to adversarially attack the best performance models of DeTraMe-Net and original CNN over 3 datasets. As shown in Table \ref{tab:UAP}, the fooling rate of DeTraMe-Net is greatly reduced by more than half compared to the original CNN one. Moreover, by attacking PlainNet, Fig. \ref{fig:advattack} shows that while increasing the adversarial attack magnitudes, our DeTraMe-PlainNet has a performance similar to PlainNet architecture in terms of fooling rate. While in comparing with the ResNet architecture, DeTraMe-ResNet greatly reduces the fooling rate, probably by taking advantage of the firmly nonexpansiveness properties of the proximal operator in the $Q$-metric. However, the robustness of residual networks in the presence of adversarial noise, remains theoretically an open issue, and hence deserves additional future investigation.
 \begin{figure}[!htb]
     \centering
     \includegraphics[width=1\columnwidth]{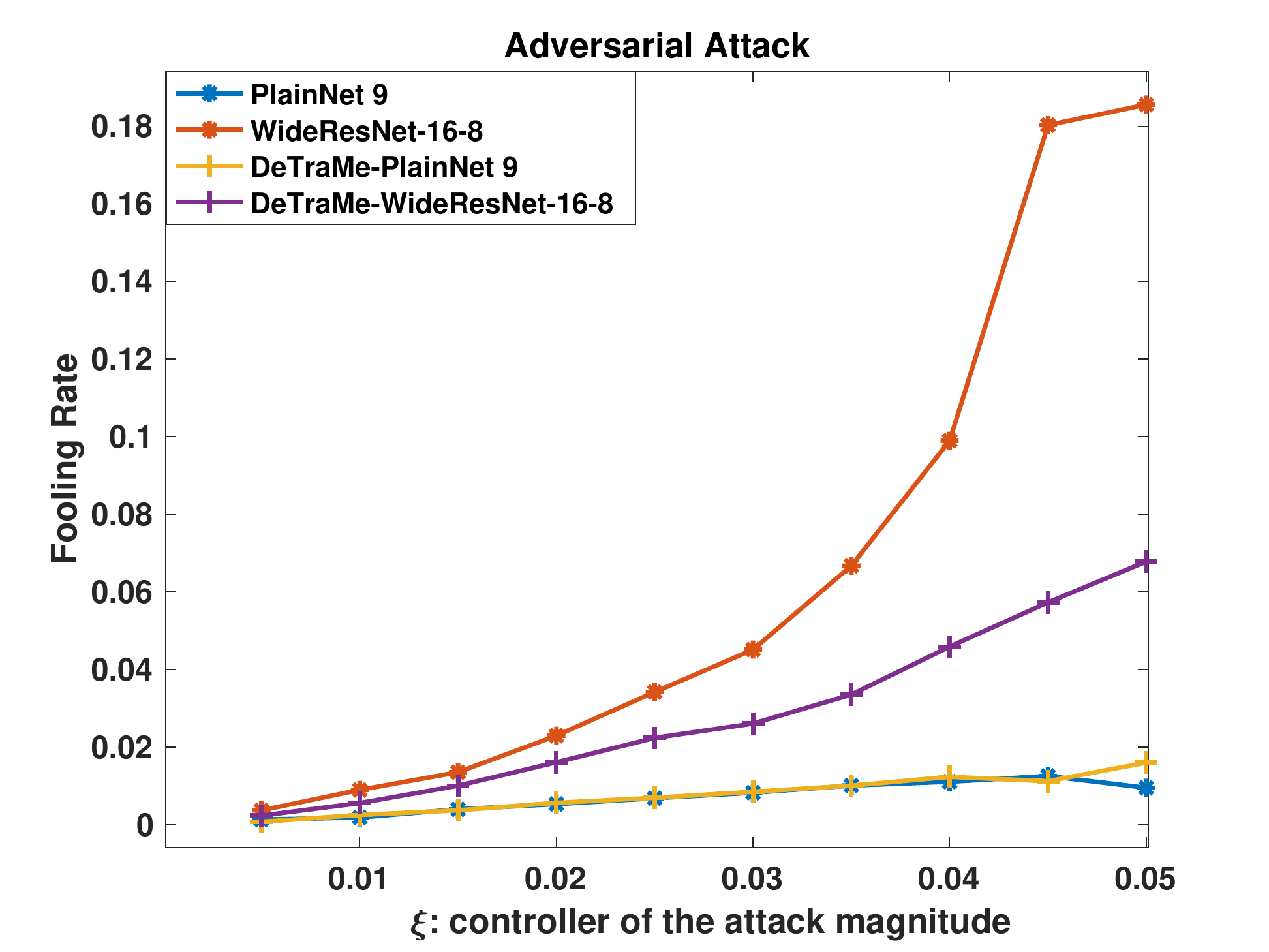}
     \caption{The fooling rate is averaged over 5 runs of CIFAR10 dataset. }
     \label{fig:advattack}
 \end{figure}

Concerning the robustness to random noise, we randomly generate a zero-mean Gaussian noise $\mathbf{v}$ and add it to the input data, where $\operatorname{E}(\|\mathbf{v}\|_2^2)=\rho\|\mathbf{x}\|_2^2$, $\rho$ controls the magnitude of random noise level with respect to the average image energy.
\begin{figure}[!htb]
     \centering
     \includegraphics[width=1\columnwidth]{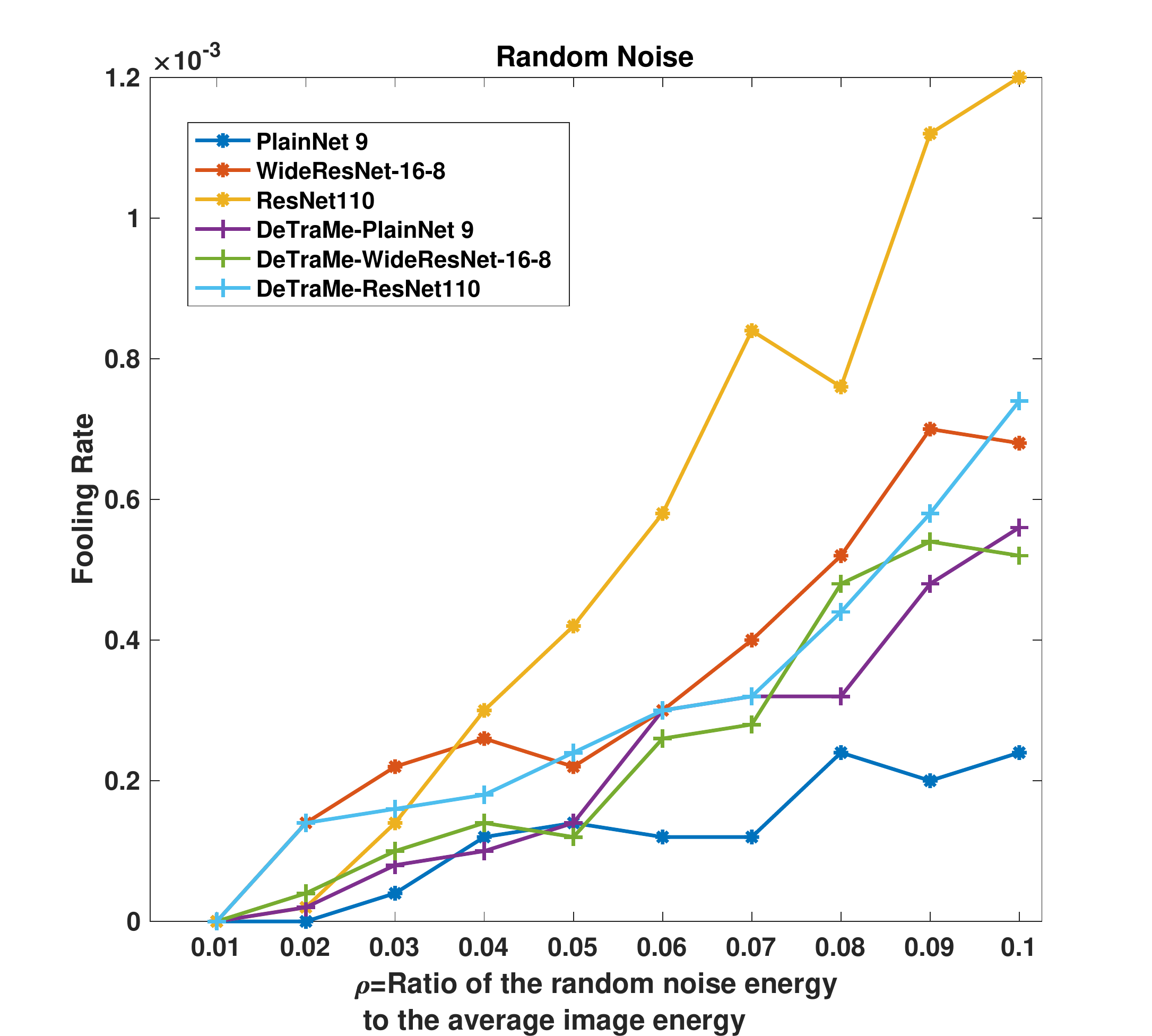}
     \caption{The fooling rate is averaged over 5 runs of CIFAR10 dataset. }
     \label{fig:randomnoise}
 \end{figure}

As shown in Fig. \ref{fig:randomnoise}, ResNet110 incurs a highest fooling rate as the noise is amplified by propagating deeply. However, WideResNet-16-8 incurs the second highest fooling rate, while our DeTraMe-ResNet110 and DeTraMe-WideResNet-16-8 achieve better performances. DeTraMe-PlainNet9 reaches a higher fooling rate than the original PlainNet, but it should be noticed that the magnitude of the fooling rate is very small, and our accuracy is about $3\%$ higher than the one of the original PlainNet CNN.

\begin{table}[!htb]
    \centering
    \resizebox{\columnwidth}{!}{
    \begin{tabular}{c|c|c|c|c}
    \hline
    \hline
        &\multicolumn{2}{c|}{Training Time (s)} & \multicolumn{2}{c}{Testing Time ($\times 10^{-4}$ s)}  \\
    \hline
      &\multicolumn{2}{|c}{CIFAR10 +}  &\multicolumn{2}{|c}{CIFAR10 +} \\
    \hline
         Network Architectures   & Original  & DeTraMe & Original  & DeTraMe \\
         \hline
         PlainNet 3-layer & 2964.45& 5837.72 & 1.93&2.50\\
         PlainNet 6-layer& 3762.75 & 6706.76 & 1.94 &2.94 \\
         PlainNet 9-layer  &3948.24 & 7451.42 & 2.01 &3.24\\
         PlainNet 12-layer &4087.98 & 8023.02& 2.12	&3.52\\
         \hline
         ResNet 8 &3152.10 &3962.09 & 1.76 & 2.10 \\
         ResNet 20 &3840.02&5316.47 & 1.90 &2.21\\
         ResNet 56 &6411.03	&8752.61   & 2.27&3.37 \\
         ResNet 110 &7709.55&12997.66  & 3.10&4.53 \\
         \hline
         WideResNet 16-4 & 4562.02&6425.65 &2.41&2.84 \\
         WideResNet 16-8 & 7104.62 &11897.93 &3.26&5.15\\
         \hline
        \hline
    \end{tabular}
    }
    \caption{
    {
    CIFAR10 with + is trained with simple translation and flipping data augmentation. All the presented results are re-implemented and run by using the same settings.
    }
    }
    \label{tab:results1}
\end{table}

\textbf{Time complexity.} Based on the running times in Table \ref{tab:results1},
training takes almost twice as much time than for generic CNNs. This appears consistent with the fact that the number of parameters of DeTraMe-Net is twice as many than for generic CNNs. However, it is worth noting that the training can be performed off-line and
that testing can still be completed in real time,
with only a slight increase of the testing time with respect to a standard CNN, that is 100 times faster than a conventional DDL method (as shown in Table \ref{tab:ddlTimecompare1}).

\section{Conclusion}
\label{sec:conclusion}
Starting from a DDL formulation, we have shown that it
is possible to reformulate the problem in a standard optimization problem with the introduction of metrics within standard activation operators. This yields a novel Deep Transform and Metric Learning problem.
This has allowed us to show that the original DDL can be performed thanks to a network mixing linear layer and RNN algorithmic structures,
thus leading to a fast and flexible network framework for building efficient DDL-based classifiers with a higher discriminiative ability.
Our experiments show that the resulting DeTraMe-Net performs better than the original DDL approach
and state-of-the-art generic CNNs. 
We think that the bridge we established between DDL and DNN will help in further understanding and controlling these powerful tools so as to
attain better performance and properties.
It would also be interesting to explore other image processing applications and understand the scope of the proposed approach.


%

\appendices
\section{Alternative Derivation of Algorithm 1} \label{sec:appendix-A}
We have presented in our paper a simple approach for
deriving the recursive model:
\begin{equation}\label{equ:u-update-supp}
\Ub_{t+1}=
  \operatorname{ReLU}\big((\mathbf{h} \mathbf{1}^\top)\odot \Zb+\widetilde{\Wb}(\Ub_t-\Zb)-\mathbf{b} \mathbf{1}^\top\big),
\end{equation}
in order to compute
\begin{equation}\label{equ:proximal-operator-supp}
    \operatorname{prox}^{\Qb}_{\lambda \psi}(\Zb)=\argmind{\Ub \in \mathbb{R}^{k\times N}} \frac{
    1}{2}\|\Ub-\Zb\|_{F,\Qb}^2+\lambda \psi(\Ub).
\end{equation}
We propose an alternative approach which is based on the classical forward-backward algorithm for solving the nonsmooth convex optimization problem in \eqref{equ:proximal-operator-supp}. The $t$-th iteration of the preconditioned form of this algorithm
reads
\begin{equation}\label{e:FBprecond0}
\Ub_{t+1}=
\operatorname{prox}_{\gamma \lambda \psi}^{\boldsymbol \Theta}
(\Ub_t-\gamma {\boldsymbol \Theta}^{-1}\Qb(\Ub_t-\Zb))
\end{equation}
where $\gamma$ is a positive stepsize and ${\boldsymbol\Theta}$
is a preconditioning symmetric definite positive matrix, and $\Ub_0 \in \mathbb{R}^{k\times N}$. The algorithm is guaranteed to converge to the solution to \eqref{equ:proximal-operator-supp} provided that
\begin{equation}
    \gamma < \frac{2}{\|{\boldsymbol \Theta}^{-1/2} \Qb {\boldsymbol \Theta}^{1/2}\|_{\rm S}},
\end{equation}
where $\|\cdot\|_{\rm S}$ denotes the spectral norm.
Eq. \eqref{e:FBprecond0} can be reexpressed as
\begin{equation}\label{e:FBprecond}
\Ub_{t+1}=
\operatorname{prox}_{\gamma \lambda \psi}^{\boldsymbol \Theta}
\big((\mathbf{I}-\gamma {\boldsymbol \Theta}^{-1}\Qb)
(\Ub_t-\Zb)+\Zb\big).
\end{equation}
Assume now that ${\boldsymbol \Theta}$ is a diagonal matrix
$\operatorname{Diag}(\theta_1,\ldots,\theta_k)$ where, for every
$i\in\{1,\ldots,k\}$, $\theta_i > 0$.
When the sparsity promoting penalization is chosen equal to
\begin{equation}
\psi=\|\cdot\|_1+\iota_{[0,+\infty)^{k\times N}}+\frac{\beta}{2\lambda}\|\cdot\|_{F}^2,
\end{equation}
the proximity operator involved in \eqref{e:FBprecond}
simplifies as
\begin{multline}\label{e:proxvect}
\forall \Ub = (u_{i,j})_{1\le i \le k,1 \le j \le N} \in \mathbb{R}^{k\times N}, \\
\operatorname{prox}_{\gamma \lambda \psi}^{\boldsymbol \Theta}
= \big(\operatorname{prox}_{\gamma \lambda \theta_i^{-1} \rho}(u_{i,j})\big)_{1\le i \le k,1 \le j \le N}
\end{multline}
where $\rho=\lambda |\cdot|+\iota_{[0,+\infty)}+\frac{\beta}{2} (\cdot)^2$.
In addition, for every $u\in \mathbb{R}$ and
$i\in \{1,\ldots,k\}$,
\begin{equation}
\operatorname{prox}_{\gamma \lambda \theta_i^{-1} \rho}(u)
= \argmind{v\in [0,+\infty)}
\frac{\theta_i}{2} (v-u)^2+
\gamma \big(\lambda |v|+\frac{\beta}{2}v^2\big).
\end{equation}
After some simple algebra, this leads to
\begin{equation}\label{e:ReLUprox}
\operatorname{prox}_{\gamma \lambda \theta_i^{-1} \rho}(u)
= \operatorname{ReLU}\Big(\frac{\theta_i}{\theta_i+\gamma \beta} u-
\frac{\gamma\lambda}{\theta_i+\gamma \beta}\Big).
\end{equation}
Altogether \eqref{e:FBprecond}, \eqref{e:proxvect}, and
\eqref{e:ReLUprox} allow us to recover an update equation
of the form \eqref{equ:u-update},
where
\begin{equation}\label{e:reparam-supp}
\begin{split}
    &\widetilde{\mathbf{W}}= ({\boldsymbol \Theta}+\gamma \beta \mathbf{I})^{-1}({\boldsymbol \Theta}-\gamma \Qb),\\
    &\mathbf{h}=\left(\frac{\theta_i}{\theta_i+\gamma \beta}\right)_{1\leq i \leq k},\\
    &\mathbf{b}=\left(\frac{\gamma \lambda}{\theta_{i}+\gamma\beta}\right)_{1\leq i \leq k}.
\end{split}
\end{equation}
Note that, if $\gamma =1$ and, for every $i\in \{1,\ldots,k\}$,
$\theta_i= q_{i,i}$, $\widetilde{\mathbf{W}}$ is a matrix with zeros on its main diagonal.

\section{Illustration of DeTraMe-Net Architectures}\label{sec:appendix-B}
\subsection{DeTraMe-PlainNet}
To replace all the RELU activation layers in PlainNet with Q-Metric ReLU leads to DeTraMe-PlainNet. Since all the RELU layers are replaced by Q-Metric ReLu, DeTraMe-PlainNet becomes equivalent to DDL.
\begin{figure}[!htb]
    \centering
    \includegraphics[width=0.5\columnwidth]{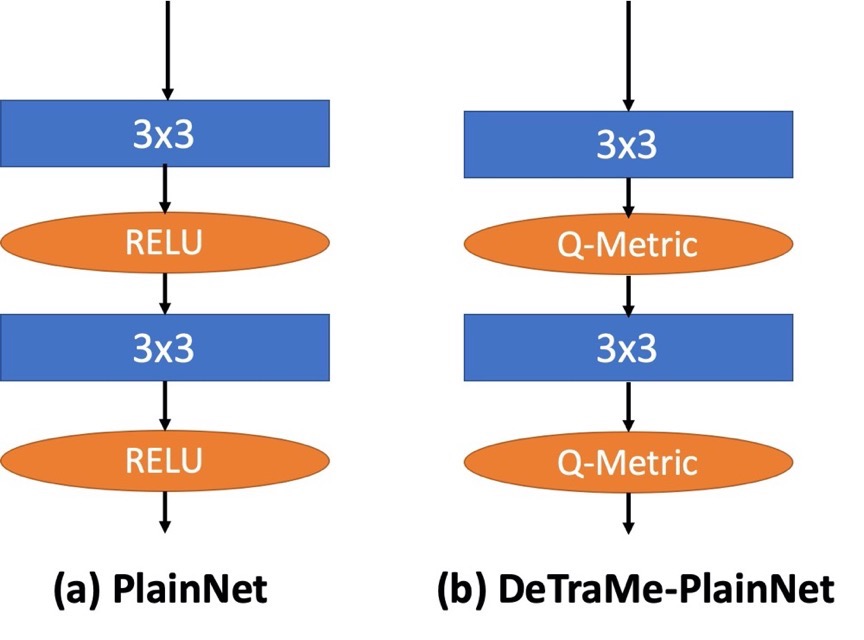}
    \caption{Architectures of PlainNet vs. DeTraMe-PlainNet}
    \label{fig:PlainQmetric}
\end{figure}

\subsection{DeTraMe-ResNet}
Replacing the RELU layer inside the block in ResNet by Q-Metric ReLU, allows us to build a new structure called DeTraMe-ResNet.
\begin{figure}[!htb]
    \centering
    \includegraphics[width=0.5\columnwidth]{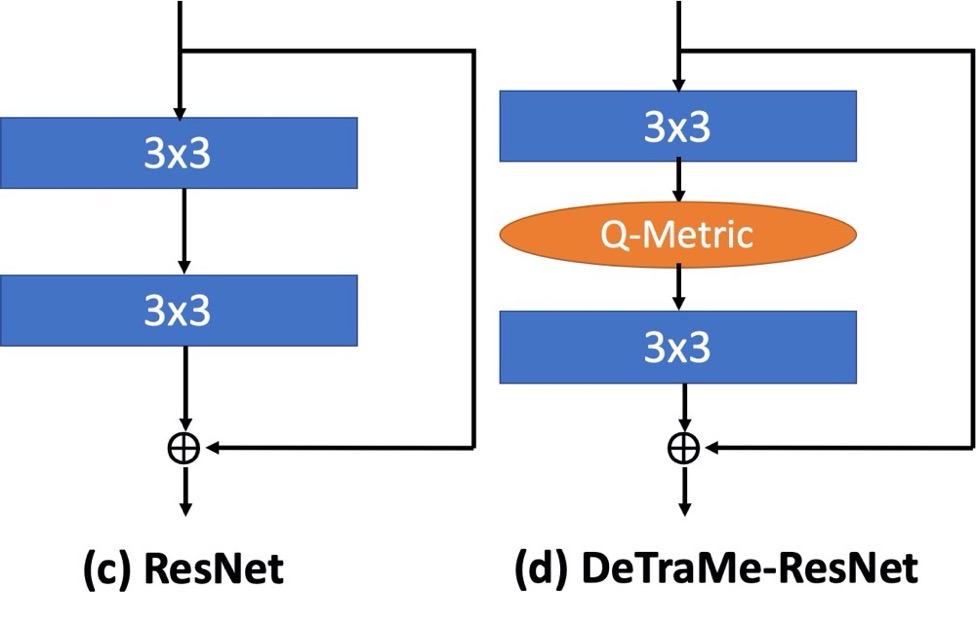}
    \caption{Architectures of ResNet vs. DeTraMe-ResNet}
    \label{fig:ResQmetric}
\end{figure}

In our experiments, for ResNet architecture, the RNN part accounting for Q-Metric learning, makes use of $3 \times 3$ filters.



\ifCLASSOPTIONcaptionsoff
  \newpage
\fi



%


\bibliographystyle{IEEEtran}
\bibliography{egbib}

%








\end{document}